\documentclass{article}

\usepackage[final]{neurips_2022}

\usepackage[normalem]{ulem}
\usepackage[utf8]{inputenc} 
\usepackage[T1]{fontenc}    
\usepackage[]{hyperref}       
\usepackage{url}            
\usepackage{booktabs}       
\usepackage{amsfonts}       
\usepackage{nicefrac}       
\usepackage{microtype}      
\usepackage{xcolor}

\usepackage[disable]{todonotes}
\usepackage{algorithm, algpseudocode}
\usepackage{amsthm}
\usepackage{amsmath}
\RequirePackage{xspace}

\newcommand{\EV}{\mathbb{E}}
\newcommand{\Reals}{\mathbb{R}}
\newcommand{\TS}{\textsc{TS}\xspace}

\newcommand{\bp}{\overline{p}}
 \newcommand{\X}{\mathcal{X}}
\newcommand{\F}{\mathcal{F}}
\newcommand{\real}{\mathbb{R}}

\newcommand{\DD}{\mathcal{D}}

\newcommand{\trace}[1]{\mbox{tr}\left(#1\right)}
\newcommand{\DKL}[2]{\DD_{\text{KL}}\left(#1\middle\|#2\right)}
\newcommand{\II}[1]{\mathbb{I}_{\left\{#1\right\}}}
\newcommand{\PP}[1]{\mathbb{P}\left[#1\right]}

\newcommand{\EE}[1]{\mathbb{E}\left[#1\right]}

\newcommand{\EEs}[2]{\mathbb{E}_{#2}\left[#1\right]}

\newcommand{\PPt}[1]{\mathbb{P}_t\left[#1\right]}
\newcommand{\EEt}[1]{\mathbb{E}_t\left[#1\right]}

\newcommand{\PPc}[2]{\mathbb{P}\left[#1\left|#2\right.\right]}

\newcommand{\PPcc}[2]{\mathbb{P}\left[\left.#1\right|#2\right]}
\newcommand{\PPcct}[2]{\mathbb{P}_t\left[\left.#1\right|#2\right]}

\newcommand{\EEc}[2]{\mathbb{E}\left[#1\left|#2\right.\right]}

\newcommand{\EEcct}[2]{\mathbb{E}_t\left[\left.#1\right|#2\right]}

\def\argmin{\mathop{\mbox{ arg\,min}}}

\newcommand{\ra}{\rightarrow}

\newcommand{\siprod}[2]{\langle#1,#2\rangle}
\newcommand{\iprod}[2]{\left\langle#1,#2\right\rangle}

\newcommand{\norm}[1]{\left\|#1\right\|}
\newcommand{\abs}[1]{\left|#1\right|}

\newcommand{\ev}[1]{\left\{#1\right\}}
\newcommand{\pa}[1]{\left(#1\right)}
\newcommand{\bpa}[1]{\bigl(#1\bigr)}
\newcommand{\Bpa}[1]{\Bigl(#1\Bigr)}

\newcommand{\wh}{\widehat}
\newcommand{\wt}{\widetilde}

\newcommand{\loss}{\ell}

\newcommand{\bloss}{\overline{\loss}}

\newcommand{\htheta}{\wh{\theta}}

\newcommand{\ttheta}{\wt{\theta}}

\newcommand{\transpose}{^\mathsf{\scriptscriptstyle T}}

\definecolor{PalePurp}{rgb}{0.66,0.57,0.66}

\newcommand{\IR}{lifted information ratio\xspace}

\newtheorem{thm}{Theorem}
\newtheorem{lemma}{Lemma}
\newtheorem{cor}{Corollary}
\newtheorem{proposition}{Proposition}
\newtheorem{assumption}{Assumption}

\title{Lifting the Information Ratio: \\An Information-Theoretic Analysis of\\ Thompson Sampling for Contextual  Bandits}

\author{%
  Gergely Neu \\
  Universitat Pompeu Fabra\\
  Barcelona, Spain \\
  \texttt{gergely.neu@gmail.com}
  \And
  Julia Olkhovskaya \\
  Vrije Universiteit Amsterdam \\
  Amsterdam, the Netherlands\\
  \texttt{julia.olkhovskaya@gmail.com}
  \And
  Matteo Papini \\
  Universitat Pompeu Fabra \\
  Barcelona, Spain \\
  \texttt{matteo.papini@upf.edu}
  \And
  Ludovic Schwartz \\
  Universitat Pompeu Fabra \\
  Barcelona, Spain \\
  \texttt{ludovic.v.schwartz76@gmail.com}
}

\begin{document}

\maketitle

\begin{abstract}
  We study the Bayesian regret of the renowned Thompson Sampling algorithm in contextual bandits with binary losses and adversarially-selected contexts. We adapt the information-theoretic perspective of~\cite{RvR16} to the contextual setting by considering a lifted version of the information ratio defined in terms of the unknown model parameter instead of the optimal action or optimal policy as done in previous works on the same setting. This allows us to bound the regret in terms of the entropy of the prior distribution through a remarkably simple proof, and with no structural assumptions on the likelihood or the prior. The extension to priors with infinite entropy only requires a Lipschitz assumption on the log-likelihood. An interesting special case is that of logistic bandits with $d$-dimensional parameters, $K$ actions, and Lipschitz logits, for which we provide a $\widetilde{O}(\sqrt{dKT})$ regret upper-bound that does not depend on the smallest slope of the sigmoid link function.
\end{abstract}

\section{Introduction}
Thompson sampling is one of the most popular algorithms for sequential decision making under uncertainty. First 
proposed by \citet{Tho33}, it has been rediscovered several times over the consequent decades and has been eventually 
popularized in the machine learning literature by \citet{CL11} and \citet{Sco10}, who pointed out its excellent 
empirical performance for solving contextual bandit problems. These empirical studies were followed by a 
sequence of breakthroughs on the front of theoretical analysis, spearheaded by the works of \citet{AG12,AG13,AG13b}, 
\citet{KKM12}, and \citet{RvR14,RvR16}. Thanks to these successes, Thompson sampling has become one of 
the gold-standard methods for solving multi-armed bandit problems. Indeed, in the last decade, several 
Thompson-sampling-style methods have been developed and analyzed for a variety of problem settings. 

The variety of different analysis techniques applied to Thompson sampling is perhaps even larger than the variety of 
problem settings that it has been applied to. The first key tools for analyzing the Bayesian regret of Thompson sampling 
for multi-armed bandits have been developed by \citet{RvR14}, and our analysis naturally borrows several of these tools. 
The worst-case results developed in said work were refined in~\citep{RvR16}, where they proved for the first time 
``information-theoretic'' bounds on the regret of TS that scale with the Shannon entropy of the optimal action $A^*$ 
under the prior on the model parameters. This result has inspired a range of follow-up works, including extensions to 
uncountable action sets \citep{DvR18}, approximate implementations \citep{LvR17,QWZLvR22}, and even new algorithms based 
on the analysis technique itself \citep{RvR17,KK18}. One limitation that this technique could not overcome so far is not 
being able to satisfyingly deal with \emph{context}.

When considering i.i.d.~contexts and finite policy classes, one can apply the theory of \citet{RvR16} and treat 
policies as actions (as done by \citealt{LZS18,BS20}), or adapt the information-ratio directly (as done by \citealt{KLK20,HLQ22}) to obtain regret bounds scaling with the entropy of the optimal policy. This approach, however, can lead to a polynomial dependence on the number of contexts when assuming no further structure about the rewards, as the entropy of the optimal policy may be as large as $\log |\Pi| = |\mathcal{X}| \log |\mathcal{A}|$.
Another variant of Thompson sampling demonstrating a similar prior-dependent regret bound has been proposed by \citet{Li13}, whose regret 
guarantees also suffer from a suboptimal dependence on the number of rounds $T$. A much more satisfying solution has 
been recently given by \citet{Zha21}, whose ``feel-good Thompson sampling'' method guarantees both frequentist and 
Bayesian regret bounds of the order $\sqrt{KT\log N}$, where $K$ is the number of actions, $T$ the number of rounds, 
and $N$ the support size of the prior distribution on model parameters $\theta^\star$. Under a Lipschitzness 
assumption on the prior and the likelihood, \citeauthor{Zha21} proves a frequentist regret bound scaling with the log-prior-probability mass assigned to the true parameter $\theta^\star$. The techniques involved in proving these results 
drew substantial inspiration from the works of \citet{FR20,FK21} and in fact results of a similar flavor were also 
proved recently in \citet{FKQR21}.

Our own approach can be seen as a reconciliation of the analysis style of \citet{Zha21} with the information-theoretic 
methodology of \citet{RvR16}. Our main conceptual contribution is proposing an adjustment of the now-classic notion of 
``information ratio'' proposed by \citet{RvR16} that applies to contextual bandits. In its original definition,  
the information ratio quantifies the tradeoff between incurring low regret  and gaining information about the optimal 
action. As we will argue, this notion of information gain is inappropriate for contextual bandits with non-i.i.d.~contexts. We propose a variant that measures the amount of information gained about the true model parameter $\theta^\star$ instead of the optimal action (which may be context dependent). The complexity notion resulting from this extension is called the ``lifted information ratio''. Our analysis shows that the Bayesian regret of Thompson sampling can be bounded in terms of the lifted information ratio and the Shannon entropy of the hidden parameter $\theta^\star$, which mirrors the result of \citet{RvR16}. Along the way, we draw inspiration from the recently proposed analysis technique of \citet{Zha21} for contextual bandits, and in fact we show that our notion of lifted information ratio bridges the concept of ``decoupling coefficient'' proposed by \citeauthor{Zha21} with the information ratio of \citeauthor{RvR16}.

We state our main results in the context of $K$-armed contextual bandits with binary losses. For countable parameter 
spaces, we prove that the Bayesian regret of Thompson sampling satisfies a bound of order $\sqrt{KTH(\theta^*)}$, where 
$H$ denotes the Shannon entropy. This result is comparable to the bound of \citet{FK21} for the FastCB algorithm, which is of the order 
$\sqrt{KT\log |\Theta|}$ and holds in a frequentist sense. This is the best result we are aware of for this setting. To 
demonstrate the flexibility of our technique, we provide an extension to logistic bandits with 
Lipschitz continuous logits, generalizing the well-studied setting involving logits that are linear functions of the 
context and the parameter $\theta$. For this setting, we prove a regret bound of order $\sqrt{KT \log 
\mathcal{N}_{1/CT}(\Theta,\norm{\cdot})}$, where $\mathcal{N}_{\varepsilon}(\Theta,\norm{\cdot})$ is the 
$\varepsilon$-covering number of $\Theta$ under norm $\norm{\cdot}$ and $C$ is the Lipschitz constant of the logits. 
This implies a regret bound of order $\sqrt{KdT}$ in the well-studied case of linear logits.
Notably, the bound does not show any dependence 
on the smallest slope $\kappa$ of the sigmoid link function that almost all existing results for this setting suffer 
from \citep{FCGS10,KZSzLGB20,AFC21,FAJC22}. Indeed, this constant has plagued all regret bounds since the early work 
of \citet{FCGS10} and was only recently moved to lower-order terms by the breakthrough work of \citet{AFC21}. Bounds 
involving other potentially large problem-dependent constants have also been proved in the Bayesian setting by 
\citet{DvR18} and \citet{DMvR19}. To our knowledge, our bounds are the first to entirely remove this 
factor.\looseness=-1\footnote{Despite our best efforts, 
we could not verify how the bounds of \citet{Zha21} scale with problem-dependent factors in this setting, due to the 
heavy use of asymptotic notation in their proofs.}

The rest of the paper is organized as follows. After introducing the necessary technical background in 
Section~\ref{sec:preli}, we discuss matters of information gains, information ratios, and decoupling coefficients in 
Section~\ref{sec:rambling}. We state our main results and instantiate them in a variety of settings in 
Section~\ref{sec:results}. We provide the 
key ideas of the analysis in Section~\ref{sec:theory} and conclude in Section~\ref{sec:conc}.

\paragraph{Notation.} For a natural number $n\in\mathbb{N}$,  $[n]=\{1,2,\dots,n\}$ denotes the set of the 
first $n$ natural numbers. 
For $x, y \in \mathbb{R}^d$, $\siprod{x}{y}$ denotes the canonical 
scalar product of $x$ and $y$, and $\norm{x}$ the Euclidean norm of $x$. We denote the Shannon entropy of a discrete 
random variable $Z$ with probability mass function $P$ as $H(Z)=-\EE{\log(P(Z))}$. We use $\bar{0}_d$ for a 
$d$-dimensional vector of zeros and $I_d$ for the $d\times d$ identity matrix.

\section{Preliminaries}\label{sec:preli}

We consider a parametric class of contextual bandits with parameter space $\Theta$, context space $\X$, and $K$ 
actions. To each parameter $\theta\in\Theta$ there corresponds a contextual 
bandit with loss distribution $P_{\theta,x,a}$ for each context $x\in\X$ and action $a\in\mathcal{A}$, with the mean of 
the loss distribution denoted by $\ell(\theta,x,a)$. We will dedicate special attention to the case where the losses 
are binary and thus $P_{\theta,x,a}$ is a Bernoulli distribution with parameter $\ell(\theta,x,a)$.
For the main part of our theoretical analysis, we will assume $\Theta$ is either a finite set or a bounded metric 
space. 

We study the problem of regret minimization in the Bayesian setting. In this setting, the environment secretly samples 
a parameter $\theta^\star$ from a known prior distribution $Q_1$ over $\Theta$. We assume that the agent has full 
knowledge of the prior and the likelihood model $P_{\theta,x,a}$. The agent interacts with the environment for 
$T$ rounds as follows. At each round $t\in[T]$, an adaptive adversary selects a context $X_t$, possibly using 
randomization and taking into account the previous history of actions and losses, but not $\theta^\star$. The latter is a common assumption~\cite[e.g.,][]{AG13b}. After 
observing the context $X_t$, the agent selects an action $A_t\in\mathcal{A}$ (possibly using randomization) and incurs a 
binary loss $L_t\sim P_{\theta^\star, X_t, A_t}$.
The goal of the agent is to minimize the expected sum of losses. In the Bayesian setting, this is equivalent to 
minimizing the Bayesian regret, defined as follows:
\begin{equation}\label{eq:bayes_regret}
	R_T = \EE{\sum_{t=1}^{T}\pa{\loss(\theta^\star, X_t, A_t) - \loss(\theta^\star, X_t, A^\star_t)}},
\end{equation}
where $A^\star_t=\argmin_{a\in\mathcal{A}}\loss(\theta^\star,X_t,a)$ is the optimal action for round $t$, and the 
expectation in~\eqref{eq:bayes_regret} is over all sources of randomness: the initial sampling of $\theta^\star$ from 
$Q_1$, the agent's randomization over actions, and the randomness of the loss realizations. 

Furthermore, let 
$\F_t=\sigma(X_1,A_1,L_1,\dots,X_t,A_t,L_t)$ be the sigma-algebra representing the history of contexts, actions and 
losses observed by the agent up to time $t$ included. 
We use $Q_t$ to denote the distribution of the unknown parameter $\theta^\star$ conditional on the past history 
$\F_{t-1}$, and simply call it the posterior distribution.
We denote by $\pi_t(\cdot\vert X_t)$ the distribution over the agent's actions conditional on $X_t$ and 
$\F_{t-1}$, and call it the agent's policy. Finally, we will frequently use the shorthand notations $\EEt{\cdot} = 
\EEc{\cdot}{\F_{t-1},X_t}$ and $\PPt{\cdot} = \PPc{\cdot}{\F_{t-1},X_t}$.

This paper is dedicated to the study of the celebrated Thompson Sampling (\TS) algorithm, defined as follows.
At each round $t$, \TS draws a parameter $\theta_t$ from the posterior distribution $Q_t$. Then, it selects the action 
that minimizes $\ell(\theta_t,X_t,\cdot)$. Finally, it updates $Q_t$ via Bayes' rule, 
obtaining the new posterior $Q_{t+1}$. The algorithm can be equivalently defined as a method that plays actions 
according to their posterior probability of being optimal, that is: $\PP{A_t=a | \F_{t-1}, X_t} = 
\PP{A^\star_t=a|\F_{t-1}, X_t}$. The pseudocode is shown as Algorithm~\ref{alg:ts}. 

\begin{algorithm}
	\caption{Thompson Sampling (\TS)}
	\label{alg:ts}
	\begin{algorithmic}[1]
		\State \textbf{Input:} prior $Q_1$
		\For{$t=1,2,\dots,T$}
		\State Sample $\theta_t \sim Q_t$
		\State Observe context $X_t$
		\State Play action $A_t=\arg\min_{a\in\mathcal{A}} \ell(\theta_t,X_t,a)$
		\State Incur loss $L_t$
		\State Update $Q_t$ with $(X_t,A_t,L_t)$ to obtain $Q_{t+1}$ by Bayes' rule
		\EndFor
	\end{algorithmic}
\end{algorithm}

\section{Regret, information ratio, and decoupling coefficient}\label{sec:rambling}
The classic results of \citet{RvR16} establish that the regret of Thompson sampling in non-contextual multi-armed 
bandit problems can be upper bounded in terms of a quantity called the \emph{information ratio}. Informally, the 
information ratio measures the tradeoff between achieving low regret and gaining information about the identity of the 
optimal action $A^\star$ (which is a deterministic function of $\theta^\star$ in the standard multi-armed bandit 
setting). The 
formal definition is given by 
\[
 \rho^*_t = \frac{\pa{\EEt{\ell(\theta^\star,A_t) - \ell(\theta^\star,A^\star)}}^2}{I_t(A^\star; (A_t,L_t))},
\]
where $I_t(A_t^\star; (A_t,L_t))$ denotes the mutual information between $A^\star$ and the action-observation pair 
$(L_t,A_t)$, 
conditioned on the history $\F_{t-1}$. Intuitively, having small information ratio implies that every time Thompson 
sampling suffers large regret, it has to gain a lot of information about the optimal action, which suggests that it 
should be possible to bound the total regret by the total amount of information that there is to be gained. The result 
of \citet{RvR16} confirms this intuition by showing that the regret of Thompson sampling is of the order 
$\sqrt{\rho^* T H(A^\star)}$, where $\rho^* = \max_t \rho^*_t$ and $H$ is the Shannon entropy.
The information ratio itself can always be upper bounded by $K/2$, but better bounds can be shown when the loss function 
has favorable structural properties (e.g., whenever the reward function is a $d$-dimensional linear function, the 
information ratio is at most $d/2$).

While this result and the underlying information-theoretic framework is very elegant, it is inappropriate for studying 
contextual bandit problems. The specific challenge is that the optimal action $A_t^\star$ changes from round to round 
and gaining large amount of information about $A_t^\star$ for any given round may not necessarily be useful for 
predicting 
future actions. To see this, consider a stylized example with action set $\ev{1,2,\dots,K}$, where there exists an 
action $a_{\text{reveal}}=1$ whose loss entirely reveals the identity of the optimal action $a^\star(x)$ for context 
$x$: 
$\ell(\theta^\star,x,a_{\text{reveal}}) = 1-10^{-a^\star(x)}$. Taking this action provides large information gain 
about $A_t^\star$, but results in large regret and reveals nothing about the future losses. Thus, in the contextual setting, one can keep following a policy that provides low information ratio while suffering linear regret. This issue necessitates an alternative definition that still permits an effective information-theoretic analysis.

Our proposition is to consider a relaxed definition of the information ratio based on the mutual information between 
the \emph{true parameter $\theta^\star$} and the observed loss. In particular, we define 
\begin{equation}\label{eq:RIR}
 \rho_t = \frac{\pa{\EEt{\ell(\theta^\star,X_t,A_t) - \ell(\theta^\star,X_t,A_t^\star)}}^2}{I_t(\theta^\star; L_t)},
\end{equation}
where $I_t(\theta^\star; L_t) = \EEt{\DKL{P_{L_t|\theta^\star,\F_{t-1},X_t,A_t}}{P_{L_t|\F_{t-1},X_t,A_t}}}$ is the 
mutual 
information between $\theta^\star$ and $L_t$, conditioned on the history $\F_{t-1}$ and the context-action pair 
$X_t,A_t$. 
This quantity measures the information that the agent gains about $\theta^\star$. High values of $I_t(\theta^\star; 
L_t)$ intuitively allow making better predictions about the future loss realizations for all possible context sequences. 
Since $A_t^\star$ is a deterministic function of $\theta^\star$ given $X_t$, the data processing inequality implies that 
the information gain about $\theta^\star$ is always greater than that about $A_t^\star$, which in turn implies that $\rho_t$ is smaller than what one would 
obtain by directly generalizing the definition of \citet{RvR16}. As this notion of information gain measures the 
efficiency of inferring the identity of a hidden parameter, we refer to $I_t(\theta^\star; L_t)$ as the \emph{lifted 
information gain}, and $\rho_t$ as defined in Equation~\eqref{eq:RIR} as the \IR. As our analysis will establish, a bounded \IR guarantees low regret, and we will show that the ratio itself can be bounded reasonably under conditions similar to the ones required by the analysis of \citet{RvR16}, despite the fact that all quantities are defined conditionally on a sequence of non-i.i.d.~contexts.

We are not the first to consider the \IR: very similar quantities have been considered in the literature on \emph{information-directed sampling}, as initiated by~\citet{RvR17}, who have already introduced an the non-contextual counterpart of the same quantity in their Section 9.2. The recent works of \citet{KLK20,HLQ22} also consider versions of the standard information ratio for contextual bandits, but their analysis crucially relies on the assumption that the contexts are i.i.d. Similarly, \citet{LvRDIOW21,LDvRX22} define versions of the information ratio using various \emph{targets} for information-directed sampling, which recovers our definition of the \IR when setting the target as $\theta^\star$. While these works do not distinguish between all these versions of the information ratio on the level of terminology, we believe that emphasizing the ``lifted'' nature of our information ratio provides more clarity in that it indicates the specific features of contextual bandit learning with parametric losses in a transparent way. Thus, we will keep using this terminology throughout the paper.

Our \IR is also closely related to a quantity appearing in the analysis of \citet{Zha21}, called 
the ``decoupling coefficient''. Adapted to our Bayesian setting, this coefficient can be defined as the smallest 
constant $\delta$ such that the following inequality holds:
\begin{align*}
 \EEt{\ell(\theta^\star,X_t,A_t) - \ell(\theta^\star,X_t,A^\star_t)} &\le \inf_{\mu > 0} 
\EEt{\mu \pa{\loss(\theta^\star,X_t,A_t) - \EEcct{\ell(\theta^\star,X_t,A_t)}{X_t,A_t}}^2 + \frac{\delta}{\mu}}
\\
&= 2\sqrt{\delta \EEt{\pa{\loss(\theta^\star,X_t,A_t) - \EEcct{\ell(\theta^\star,X_t,A_t)}{X_t,A_t}}^2}},
\end{align*}
where the first line gives the original definition mirroring that of \citet{Zha21} and the second line plugs in the 
choice of $\mu$ achieving the infimum. Reordering gives the value of the optimal $\delta^*_t$:
\[
 \delta^*_t = \frac{\pa{\EEt{\ell(\theta^*,X_t,A_t) - 
\ell(\theta^*,X_t,A_t^*)}}^2}{4\EEt{\pa{\loss(\theta^*,X_t,A_t) - \EEcct{\loss(\theta^*,X_t,A_t)}{X_t,A_t}}^2}},
\]
which matches our definition of the \IR, up to the difference of replacing the mutual information 
by the root mean-squared error in predicting the true parameter $\theta^*$. Notably, this definition essentially
coincides with the \IR for the special case of Gaussian losses.\looseness-1

\section{Main results}\label{sec:results}
In this section, we state our main results concerning the Bayesian regret of Thompson sampling for contextual bandits. 
We will assume that the losses are binary and the action space is finite, unless otherwise stated. However, several 
of our results can be generalized beyond this setting. We will illustrate this in Section~\ref{sec:gaussian}, where 
we provide some additional results for the classic setting of Gaussian linear contextual bandits.

We begin by stating two general regret bounds in terms of the \IR defined in Equation~\eqref{eq:RIR}. 
The reader that is not interested in the full generality of our theory may skip to Section~\ref{sec:concrete} for concrete regret bounds.
Our first abstract bound applies to priors with finite entropy, the simplest example being finite parameter spaces. 
\begin{thm}\label{th:main_discrete}
	Assume $Q_1$ is supported on the countable set $\Theta_1 \subseteq \Theta$ and that the \IR for all rounds $t$ satisfies $\rho_t \le \rho$ for some $\rho > 0$. Then, the Bayesian regret of \TS after $T$ 
	rounds can be bounded as
	\begin{equation*}
	R_T \le \sqrt{\rho TH(\theta^\star)}.
	\end{equation*}
	 In particular if $\Theta_1$ is a finite set with $|\Theta_1|=N$, the regret of \TS satisfies
	\begin{equation*}
	R_T \le \sqrt{\rho T\log N}. 
	\end{equation*}
\end{thm}
The proof of this theorem is stated in Section~\ref{sec:proof_thm1}. 
Unfortunately, the Shannon entropy can be unbounded for distributions with infinite support, which is in fact the 
typical situation that one encouters in practice. 
To address this concern, we develop a more general result, that holds for a broader family of distributions. 
In the following, $(\Theta, \varrho)$ is a metric space with metric $\varrho:\Theta^2\to\Reals$.
We make the following regularity assumption on the likelihood function $P_{\theta,x,a}$:
\begin{assumption}\label{ass:log_lipschitz}
	There exists a constant $C>0$ such that for any  $\theta, \theta' \in \Theta_1$, 
	$\abs{\log P_{\theta,x,a}(L) - \log P_{\theta',x,a}(L)} \le C \varrho(\theta, \theta')$ holds for all 
	$x\in\X$, $a\in\mathcal{A}$, and $L \in \{0,1\} $. 
\end{assumption}
Under this assumption, we can state a variant of Theorem~\ref{th:main_discrete} that applies to metric parameter 
spaces:
\begin{thm}\label{th:main_continuous}
	Assume $(\Theta, \varrho)$ is a metric space, and $Q_1$ is supported on $\Theta_1\subseteq\Theta$ with 
	$\varepsilon$-covering number $\mathcal{N}_{\varepsilon}(\Theta_1,\varrho)$. Let Assumption~\ref{ass:log_lipschitz} hold, and 
	assume the \IR for all rounds $t$ satisfies $\rho_t \le \rho$ for some 
	$\rho > 0$. Then, the Bayesian regret of \TS after $T$ rounds can be bounded as
	\begin{equation*}
	R_T \le \sqrt{\rho 
T\textstyle{\min_{\varepsilon}}\left(\log\left(\mathcal{N}_{\varepsilon}(\Theta_1,\varrho)\right)+2\varepsilon 
CT\right)}.
	\end{equation*}
\end{thm}
The proof is based on a covering argument on top of the proof of the previous theorem, and is provided in 
Appendix~\ref{app:th2}. To get a better understanding of Assumption~\ref{ass:log_lipschitz}, it is useful to notice that 
it is satisfied in 
basic settings, like logistic bandits with Lipschitz logits. See Section~\ref{sec:concrete} for details.

\subsection{Bounding the \IR}
At this point, some readers may worry that the \IR may be impossible to bound due to the lifting 
to the space of parameters $\Theta$. To address this concern, we now turn to showing bounds on the \IR.
We first consider the unstructured case, that holds for arbitrary parameter spaces, likelihoods, and priors.
\begin{lemma}\label{lem:ir_finite}
Suppose that the losses are binary and $|\mathcal{A}|=K$. Then, the 
\IR of Thompson sampling satisfies $\rho_t \le 2K$ for all $t\ge 1$.
\end{lemma}
The proof of this lemma (provided in Section~\ref{sec:proof_ir_finite}) relies on a decoupling argument between the choice of 
the action and of the parameters at round $t$, inspired by \citet{Zha21}. Taking his argument one step further, we 
center our analysis around an application of convex conjugacy, which we believe may be applicable in a broader variety 
of settings. We wish to highlight that this proof technique is very different from the information-theoretic 
methodology pioneered by \citet{RvR16}.

Next, we consider the case of linear expected losses in Euclidean parameter spaces, which, in principle, allows for an 
unbounded number of actions.
\begin{lemma}\label{lem:ir_linear}
Suppose that $\Theta\subseteq\real^d$, the losses are binary, and the expected losses are linear functions of the form 
$\ell(\theta,x,a) = \iprod{\theta}{\varphi(x,a)}$, where $\varphi:\mathcal{X}\times\mathcal{A}\ra \real^d$ is a feature 
map, such that $\ell(\theta,x,a) \in [0,1]$ for all $x\in \mathcal{X}, a\in \mathcal{A}, \theta \in \Theta$. Then, the \IR satisfies $\rho_t \le d/2$.
\end{lemma}
The proof of this result (deferred to Appendix~\ref{app:lir_linear}) follows the arguments of 
Proposition~5 of \citet{RvR16}, with some small but important changes to account for the presence of contexts. 

Notably, both of these results match the classic bounds of~\citet{RvR16} on the standard definition of the 
information ratio for these settings (cf.~their Propositions~5 and~3), implying that lifting to 
the space of parameters does not substantially impact the regret-information tradeoff.

\subsection{Concrete regret bounds for Bernoulli bandits}\label{sec:concrete}
We now instantiate our bounds in two well-studied settings for Bernoulli bandits. 
We start from the fully unstructured case, assuming finite actions and finitely supported prior. The following regret bound follows directly from 
Theorem~\ref{th:main_discrete} and Lemma~\ref{lem:ir_finite}.
\begin{thm}
	Consider a contextual bandit with $K$ actions and binary losses, and suppose $\Theta_1$, the support of $Q_1$, is finite with $|\Theta_1|=N $. Then, the Bayesian regret of \TS satisfies:
	\begin{equation*}
	R_T \le \sqrt{2KT\log N}.
	\end{equation*}
\end{thm}
This result is comparable to the best known regret guarantees for this problem due to \citet{FK21} and \citet{Zha21}, 
and matches the minimax rate for unstructured contextual bandits with a policy class of size 
$N$~\citep{BLLRS11,DHKKLRZ11}. However, we are not aware of a comparable result for the Thompson sampling algorithm in 
the literature, be it Bayesian or not.

Moving to Bernoulli bandits with structure, we consider a well-studied setting known as \emph{logistic bandits}.
In this model, the losses are generated by a Bernoulli distribution as $L_t(\theta, x, a) \sim 
\mathrm{Ber}(\sigma(f_{\theta}(x,a)) )$, where $\sigma(z) = 1/(1+ e^{-z})$ is the sigmoid function. We just assume that 
$f_{\theta}(x,a)$, called the logit function, is $C$-Lipschitz in $\theta$, which directly implies that 
Assumption~\ref{ass:log_lipschitz} holds. 
Notice that our definition generalizes the commonly used notion of logistic bandits that consider linear logit 
functions of the form $f_{\theta}(x,a)=\langle\theta,\varphi(x,a)\rangle$, where 
$\varphi:\X\times\mathcal{A}\to\Reals^d$ is 
some feature map. 
Our result for logistic bandits is based on 
Theorem~\ref{th:main_continuous} and Lemma~\ref{lem:ir_finite}.
\begin{thm}\label{th:main_logistic}
	Assume $\Theta\subset\Reals^d$ and $\norm{\theta}\le S$ for all $\theta\in\Theta_1$. Consider a class of logistic 
bandits with $K$ actions and $C$-Lipschitz logit function $f_{\theta}(x,a)$. Then, the Bayesian regret of \TS after $T$ rounds can be bounded as:
	\begin{equation*}
	R_T \le \sqrt{2K T\left(d\log(2SCT + 1)+1\right)}.
	\end{equation*}
\end{thm}
The proof follows from an application of Theorem~\ref{th:main_continuous}, and can be found in 
Appendix~\ref{main_logistic_proof}. To our knowledge, this is the first regret bound for logistic bandits with 
nonlinear logits.

We can further specialize our result to linear logits, the setting that is most commonly studied in the literature:
\begin{cor}\label{th:main_linear_logit}
	Assume $\Theta\subset\Reals^d$, $\norm{\theta}\le S$ for all $\theta\in\Theta_1$. Consider a class of logistic bandits with $K$ actions and linear logit function $f_{\theta}(x,a)=\langle\varphi(x,a),\theta\rangle$, with $\norm{\phi(x,a)}\le B$ for all $x\in\X$ and $a\in[K]$. Then, the Bayesian regret of \TS after $T$ rounds can be bounded as:
	\begin{equation*}
	R_T \le \sqrt{2KT\left(d\log(2SBT+1)+1\right)} = \wt{O}\big(\sqrt{dKT}\big).
	\end{equation*}
\end{cor}

The proof of the corollary is stated at Appendix~\ref{main_logistic_proof_col}. A remarkable feature of this bound is that it shows no dependence on the minimum derivative of the sigmoid link 
function, albeit at the price of a $\sqrt{K}$ factor in the bound. Nevertheless, we believe this to be the first 
regret guarantee that entirely gets rid of this potentially enormous constant without very strong assumptions. Indeed, 
this constant has been present in nearly all previous bounds we are aware of 
\citep{FCGS10,LLZ17,FACF20,AFC21,FAJC22}---although these results have the advantage of holding in a frequentist sense.
In the Bayesian setting, the works of \citet{DvR18} and \citet{DMvR19} have proved a variety of bounds on the regret of 
Thompson sampling for non-contextual logistic bandits, but none of them are directly comparable with our result above. 
\citet{DMvR19} prove a regret bound of order $d\sqrt{T}$ for a highly specialized setting with $\mathcal{A}=\Theta$, 
and a range of other bounds under a variety of strong assumptions.

An improved bound for the many-actions setting that scales at most logarithmically with $K$ remains an open problem. 
Its difficulty is testified by a set of negative examples provided by~\citet{DMvR19}, and by a long-lived conjecture 
of~\citep{DvR18} regarding the information ratio for logistic bandits, that, to our knowledge, has not yet been 
verified in theory. 

\subsection{Beyond binary losses: linear bandits with Gaussian noise}\label{sec:gaussian}
We now illustrate how our techniques (in particular the \IR) can be extended beyond the case of binary losses, and 
in particular consider the classic setting of Bayesian linear contextual bandits, where the loss is a linear function 
of a $d$-dimensional feature map with additive Gaussian noise, and the prior is also Gaussian. 
\begin{lemma}\label{lem:ir_gaussian}
	Suppose $\theta^*$ be a $d$-dimensional normal random vector, so $\theta^*\sim \mathcal{N}\pa{\bar{0}_d, \lambda I_d} $ for some $\lambda > 0$, the loss is a normal random  variable  $L_t\sim \mathcal{N}(\ell(\theta^\star,X_t,A_t), \sigma^2)$ for some standard deviation $\sigma >0$, and the expected loss is linear,  so
	$\ell(\theta, x, a) =\langle\theta,\varphi(x,a)\rangle$, where $\varphi:\X\times[K]\to\Reals^d$ is a feature map. Then, the lifted information ratio satisfies $\rho_t\le \min\{d, 2(1+\log K)\}$.
\end{lemma}
For this setting, the bound $\rho_t \le d$ has already been shown for the classic information ratio by \citet{RvR16}. 
However, we believe that our bound $\rho_t \le 2(1 + \log K)$ is new for any definition of information ratio. 
By combining this bound on the \IR with standard arguments for linear contextual bandits, we recover both of the 
well-known regret bounds of order $d\sqrt{T\log T}$ and $\sqrt{dT\log\pa{KT}}$ for this 
seting, respectively due to \citet{APS11} and~\citet{CLRS11}. \footnote{While for the cases when $d \ll \log K$ the 
dependence on $K$ is not present in the regret bound, our analysis is restricted to the setting with the finite number 
of actions. Still, using a standard discretization argument, it is possible to extend the analysis to infinite action 
spaces.}
See  Corollary~\ref{gp_prior} in  Appendix~\ref{app:gaussian} for a rigorous statement and the proof.
This result, although not surprising, indicates once again that our notion of \IR does not lead to compromises in 
performance, even when the losses are not binary.

\section{Analysis}\label{sec:theory}
This section presents the key ideas of the proofs of our main results. 
We will just provide the proof of Theorem~\ref{th:main_discrete} and that of Lemma~\ref{lem:ir_finite}, which we believe offer the most insight into our techniques. All other proofs, included those of auxiliary lemmas, are deferred to Appendix~\ref{app:proofs} and~\ref{app:linear}.
For sake of clarity, we focus on 
the relatively simple case where $\Theta_1$ is countable, so that (with a 
slight abuse of notation) we can write $Q_t(\theta)$ to denote the posterior probability associated with $\theta$. Note, however, that our full proofs also handle the case of general distributions (details in Appendix~\ref{app:bayes_trick}).

\subsection{The proof of Theorem~\ref{th:main_discrete}}\label{sec:proof_thm1}
Recalling the definition of the \IR (Equation~\ref{eq:RIR}), we first notice that the 
regret can be rewritten as follows:
\begin{align}
 &\EE{\sum_{t=1}^T \pa{\ell(\theta^\star,X_t,A_t) - \ell(\theta^\star,X_t,A_t^\star)}} = 
 \EE{\sum_{t=1}^T \EEt{\ell(\theta^\star,X_t,A_t) - \ell(\theta^\star,X_t,A_t^\star)}} 
 \nonumber\\
 &\qquad\qquad=  \EE{\sum_{t=1}^T \sqrt{\rho_t I_t(\theta^\star; L_t)}} \le \sqrt{\EE{\sum_{t=1}^T \rho_t} 
\cdot \EE{\sum_{t=1}^T I_t(\theta^\star; L_t)}},\label{eq:cs}
\end{align}
where the first step uses the tower rule of expectation, the second step the definition of 
$\rho_t$, and the final step follows from the Cauchy--Schwarz inequality.

The key challenge is then to bound the sum of information-gain terms. The following lemma provides a more tractable 
form of this sum:
\begin{lemma}\label{lem:bayes_trick}
	Under the assumptions of Theorem~\ref{th:main_discrete},
\begin{align*}
	\EE{\sum_{t=1}^T I_t(\theta^\star; L_t) } 
	= \EE{\log \frac{\prod_{t=1}^T p_{\theta^\star,t}(L_{t})}{\sum_{\theta} Q_1(\theta) \prod_{t=1}^T 
p_{\theta,t}(L_{t}) 
}}.
\end{align*}
\end{lemma}
The proof of this lemma is based on a classic ``Bayesian telescoping'' argument that we have learned from \citet{Gru12}. We provide 
the proof of Lemma~\ref{lem:bayes_trick} in Appendix~\ref{app:bayes_trick}. 
Supposing now that the prior has bounded entropy, we can easily bound the term appearing on the right hand 
side as follows:
\begin{align*}
&\EE{ \log \frac{\prod_{t=1}^T 
p_{\theta^\star,t}(L_{t})}{\sum_{\theta} Q_1(\theta) \prod_{t=1}^T p_{\theta,t}(L_{t})}} \!\le\! \EE{  \log 
\frac{\prod_{t=1}^T p_{\theta^\star,t}(L_{t})}{Q_1(\theta^\star) \prod_{t=1}^T p_{\theta^\star,t}(L_{t})} 
 } \!=\! \EE{ -\log Q_1(\theta^\star)} \!=\! H(\theta^\star).
\end{align*}
This concludes the proof of the first statement. 
The second statement follows from the first using 
the trivial bound on the Shannon entropy of any finite-support 
distribution.

\allowdisplaybreaks
\subsection{The proof of Lemma~\ref{lem:ir_finite}}\label{sec:proof_ir_finite}
We start by introducing some notation that will be useful for the proof. In particular, we use $g$ to denote the binary 
relative entropy function defined for all $p,q\in[0,1]$ as
\begin{equation}\label{eq:bin_relent}
g(p\|q) = p\log\frac{p}{q} + (1-p)\log\frac{1-p}{1-q},
\end{equation}
and we use the convention $0\cdot\log 0 = 0$. Furthermore, we define the posterior mean loss as $\bloss_t(x,a) = 
\EEt{\loss(\theta^\star,x,a)}$. These notations allow us to conveniently rewrite the information gain as
\begin{align*}
 I_t(\theta^\star; L_t) 
 &= \EEt{\DKL{P_{L_t|\theta^\star,\F_{t-1},X_t,A_t}}{P_{L_t|\F_{t-1},X_t,A_t}}} = 
\EEt{g\pa{\loss(\theta^\star,X_t,A_t)\middle\|\bloss_t(X_t,A_t)}}.
\end{align*}
We will now prove a generalization of Lemma~\ref{lem:ir_finite}, which 
will directly imply the original result:
\begin{lemma}\label{lem:ir_finite_general}
Under the assumptions of Lemma~\ref{lem:ir_finite}, for all $t$, the \IR of Thompson sampling satisfies $\rho_t \le 2\sum_{a\in\mathcal{A}} \EEt{\bloss_t(X_t,a)}$.
\end{lemma}
\begin{proof}
The proof is based on an application of the Fenchel--Young inequality, which requires the introduction of the 
Legendre--Fenchel conjugate of $g$ with respect to its first argument. This function is defined for all $u\in\Reals$ as
\begin{equation}\label{eq:gstar_bound}
g^*(u\|q) = \sup_{p\in[0,1]} \ev{pu - g(p\|q)} = \log(1+q(e^u - 1)) \le q\left(u+\frac{u^2}{2}\right),
\end{equation}
where the second equality and the inequality follow from a set of straightforward calculations deferred to 
Appendix~\ref{app:aux}.
Turning to the actual proof, we consider the instantaneous pseudo-regret $r_t=\loss(\theta^\star X_t, A_t) - 
\loss(\theta^\star, X_t, A^\star_t)$ in a fixed round $t$ and write the following (for 
any $\eta>0$):
	\begin{align*}
		\EEt{r_t} 
		&=\EEt{\ell(\theta^\star, X_t, A_t) - \ell(\theta^\star, X_t, A^\star_t)} \nonumber\\
		&=\EEt{\ell(\theta^\star, X_t, A_t) - \ell(\theta_t, X_t, 
A_t)} \\
&\qquad\qquad\mbox{(using that $(\theta^\star,A_t^\star)$ has the same conditional distribution as 
$(\theta_t,A_t)$)}\nonumber
\\
		&=\EEt{\bar{\ell}_t(X_t, A_t) - 
\ell(\theta_t, X_t, A_t)} \\
&\qquad\qquad\mbox{(by conditional independence of $\theta^\star$ and $A_t$)}\nonumber\\
		&= \EV_t\left[\bar{\ell}_t(X_t, A_t) - 
\sum_{a\in\mathcal{A}}\mathbb{I}_{\{A_t=a\}}\frac{\eta\pi_t(a|X_t)}{\eta\pi_t(a|X_t)}\ell(\theta_t, X_t, a)\right]\nonumber\\
		&\le \EV_t\left[\bar{\ell}_t(X_t, A_t) + \eta\sum_{a\in\mathcal{A}}\pi_t(a|X_t)\bigg(g(\ell(\theta_t, X_t, 
a)\|\bar{\ell}(X_t,a)) 
\right.\nonumber\\&\qquad+\left.
g^*\left(-\frac{\mathbb{I}_{\{A_t=a\}}}{\eta\pi_t(a|X_t)}\middle\|\bar{\ell}(X_t, 
a)\right)\bigg)\right] \\
&\qquad\qquad\mbox{(by the Fenchel-Young inequality)}\nonumber\\
		&\le \EV_t\left[\bar{\ell}_t(X_t, A_t) + \eta\sum_{a\in\mathcal{A}}\pi_t(a|X_t)\left(g(\ell(\theta_t, X_t, 
a)\|\bar{\ell}(X_t,a)) - \frac{\mathbb{I}_{\{A_t=a\}}}{\eta\pi_t(a|X_t)}\bar{\ell}(X_t,a) 
\right.\right.\nonumber\\&\qquad+\left.\left.
\frac{\mathbb{I}_{\{A_t=a\}}}{2\eta^2\pi_t(a|X_t)^2}\bar{\ell}(X_t,a)\right)\right] \\
&\qquad\qquad\mbox{(by Equation~\ref{eq:gstar_bound})}\nonumber\\
		&= \EV_t\left[ \eta\sum_{a\in\mathcal{A}}\pi_t(a|X_t)g(\ell(\theta_t, X_t, a)\|\bar{\ell}(X_t,a)) + 
\frac{1}{2\eta}\sum_{a\in\mathcal{A}}\bar{\ell}(X_t,a)\right]\nonumber\\
&\qquad\qquad\mbox{(by the tower rule of expectation and $\EEt{\mathbb{I}_{\{A_t=a\}}} = \pi_t(a|X_t)$)}\nonumber\\
		&= \EV_t\left[ \eta\sum_{a\in\mathcal{A}}\pi_t(a|X_t)g(\ell(\theta^\star, X_t, a)\|\bar{\ell}(X_t,a)) + 
\frac{1}{2\eta}\sum_{a\in\mathcal{A}}\bar{\ell}(X_t,a)\right]\\
&\qquad\qquad\mbox{(using again that $\theta_t$ has the same conditional distribution as 
	$\theta^\star$)}\nonumber
\\
&=\eta I_t(\theta^*;L_t) + \frac{1}{2\eta}\sum_{a\in\mathcal{A}}\EEt{\bar{\ell}(X_t,a)}.
\end{align*}
Choosing the value of $\eta>0$ for which the latter expression is minimal, we obtain 
$\EEt{r_t}\le\sqrt{2I_t(\theta^*;L_t)\sum_{a\in\mathcal{A}}\EEt{\bar{\ell}_t(X_t,a)}}$. The proof is completed by 
taking the square on both sides and rearranging.
\end{proof}

\section{Conclusion}\label{sec:conc}
We have presented a new theoretical framework for analyzing Thompson sampling in contextual bandits, resulting in new results that advance the state of the art in the well-studied problem of logistic bandits. We believe that 
these results are encouraging and that our analytic framework may find many more applications in the future.

As always, we leave many more questions open than what we have closed. One major question regarding logistic bandits is 
if it is possible to improve our new results by significantly toning down the dependence on the number of actions $K$. 
In light of existing hardness results for nonlinear bandit problems (e.g., for generalized linear bandits with ReLU 
activation,~\citealp{DYM21,FKQR21}) we suspect that this may not be possible.
As a more modest goal, we are curious to find out if the \IR can be upper bounded in terms of 
the smallest-slope parameter $\kappa$ as done in many other works on logistic bandits since~\citep{FCGS10}.
We conjecture that a $O(\kappa^{-2}d)$ bound on the \IR is indeed possible, 
but we were not able to prove it so far. 
This is the case for the eluder dimension~\citep{RvR13}, another complexity measure that has been used to upper-bound 
the regret for contextual bandits. The eluder dimension for linear losses is $O(d)$, but for nonlinear losses we know 
only of $O(\kappa^{-2}d)$ bounds for the generalized linear case. 

More broadly, we believe that the most interesting immediate challenge is to extend our results to hold beyond the Bayesian setting. As a counterexample by \citet{Zha21} shows, this may not be possible in general, but we wonder if his ``feel-good'' adjustment of Thompson sampling could be analyzed with the techniques we introduced in this paper. 

Throughout the paper, we have studied several different settings, some of which come with a wide range of possible 
choices for the form of the prior and the likelihood. In some of this scenarios, updating the posterior and sampling 
from it may be computationally challenging. We have ignored this aspect in order to focus on the pure online-decision 
aspects, and implicitly assumed that posterior sampling can be performed without approximations. In practice, several 
heuristics have been proposed, see for instance~\citep{DFE18} on efficient \TS for logistic bandits. It would be 
interesting to study how approximate sampling affects our regret guarantees, along the line of~\citep{PAD19,MPMJB20}.

\begin{ack}
This project has received funding from the European Research Council (ERC) under the European Union’s Horizon 2020 research and innovation programme (Grant agreement No.~950180).	
\end{ack}

\bibliographystyle{abbrvnat}
\bibliography{thomson,shortconfs}



\clearpage
\appendix


\section{Omitted proofs}\label{app:proofs}
\subsection{The proof of Lemma~\ref{lem:bayes_trick}}\label{app:bayes_trick}
For didactic purposes, we provide two proofs for this lemma. We first start with the simple case of distributions with 
finite supports that allows us to spell out the steps in the proof using simple and intuitive notation. Then, we 
provide a proof for general prior distributions. Some general notations that we will use throughout are the following. 
We let $\mathcal{D}_{\theta, x, a}$ be the distribution of the loss given context $x$, 
action $a$ and a fixed parameter $\theta \in \Theta$, and let $L_{\theta, x, a}$ denote the random variable with said 
distribution. Using this notation, notice that $L_{t} \sim \mathcal{D}_{\theta^*, X_t, A_t}$. Finally, we 
define the likelihood function $p_{\theta,t}(c) = \PPcc{ L_{\theta, X_t, A_t} = 	c}{X_t,A_t}$ 

\paragraph{Proof for countably supported priors.} We first assume that the support $\Theta_1\subseteq\Theta$ is 
countable, which will allow us to reason about probability mass functions. In particular, with a slight abuse of our 
notation, we will write $Q_t(\theta) = \PPt{\theta^\star = \theta}$ (which should otherwise be written as 
$Q_t(\ev{\theta})$). Defining the Bayesian posterior predictive distribution
$\overline{p}_t(c) = \sum_{\theta} Q_t(\theta) p_{\theta,t}(c)$,
we can write
	\begin{equation}\label{eq:MI_vs_posterior_1}
	\begin{split}
		I_t(\theta^*;L_t) &=  \EEt{\DKL{P_{L_t|\theta^*,\F_{t-1},X_t,A_t}}{P_{L_t|\F_{t-1},X_t,A_t}}}  
		\\
		&=\EEt{\EEs{\log\frac{p_{\theta^\star,t}(L_t)}{\overline{p}_t(L_t)} }{ 
p_{\theta^\star,t}}}=\EEt{\log\frac{p_{\theta^\star,t}(L_t)}{\overline{p}_t(L_t)}}.	
\end{split}
\end{equation}
Then, summing up 
and taking marginal expectations, we get
	\begin{equation}\label{eq:MI_vs_posterior_2}
	\begin{split}
		&\EE{\sum_{t=1}^T I_t(\theta^*;L_t) } 
		= \EE{\sum_{t=1}^T \EEt{ \log \frac{p_{\theta^*,t}(L_t)}{\overline{p}_t(L_t)}}}
		= \EE{\sum_{t=1}^T \log \frac{p_{\theta^*,t}(L_t)}{\overline{p}_t(L_t)}}
		\\
		&= \EE{\log \frac{\prod_{t=1}^T p_{\theta^*,t}(L_{t})}{\prod_{t=1}^T \overline{p}_t(L_t)}}
\end{split}
\end{equation}

	To proceed, let us notice that the posterior updates take the following form by definition:
\begin{align*}
		Q_{t+1}(\theta) &= \frac{Q_{t}(\theta) p_{\theta,t}(L_t)}{\sum_{\theta'} Q_{t}(\theta') p_{\theta'}(L_t)}.
	\end{align*}
	Also, let us define the notation $\overline{p}(L_{1:T}) = \sum_{\theta} Q_1(\theta) \prod_{t=1}^T 
p_{\theta,t}(L_{t})$ and notice that we can express this quantity by a recursive application of the above expression 
as 
	\begin{align*}
		\overline{p}(L_{1:T}) &= \prod_{t=1}^T \frac{\overline{p}(L_{1:t})}{\overline{p}(L_{1:t-1})}
		= \prod_{t=1}^T \frac{\sum_{\theta} Q_1(\theta) \prod_{k=1}^t p_{\theta,k}(L_k)}{\sum_{\theta'} Q_{1}(\theta') 
			\prod_{k=1}^{t-1} p_{\theta',k}(L_k)}
		\\
		&= \prod_{t=1}^T \sum_\theta Q_t(\theta) p_{\theta,t}(L_t)
		= \prod_{t=1}^T \overline{p}_t(L_t).
	\end{align*}
	Then, we have
	\begin{align*}
		\log \frac{\prod_{t=1}^T p_{\theta^*,t}(L_{t})}{\prod_{t=1}^T \bp_t(L_t)}  &= \log \frac{\prod_{t=1}^T 
				p_{\theta^*,t}(L_{t})}{\overline{p}(L_{1:T})} 
		= \log \frac{\prod_{t=1}^T p_{\theta^*,t}(L_{t})}{\sum_{\theta'} Q_1(\theta') \prod_{t=1}^T 
p_{\theta',t}(L_{t})}.
	\end{align*}
\qed

\paragraph{Proof for general prior distributions.}
The proof follows from similar arguments, although we cannot work with probability mass functions any more. In 
particular, we will denote by $Q_1$ the prior distribution of $\theta^*$, which satisfies the following identity:
$$
\PP{\theta^* \in A} = \int_{\theta \in \Theta} \II{\theta \in A}dQ_1(\theta).
$$
Similarly, we denote by $Q_{t+1}$ the posterior distribution on $\theta^*$ after round $t$, which satisfies
	$$\PP{\theta^* \in A| \F_t} = \int_{\theta \in \Theta} \II{\theta \in A}dQ_{t+1}(\theta)$$
We now apply Bayes theorem for general distributions that gives the following expression for $q_{t+1}$ :
\begin{align*}
	q_{t+1}(\theta) &= \frac{q_t(\theta)p_{\theta, t}(L_t)}{\int_{\theta' \in \Theta} p_{\theta', t}(L_t)q_t(\theta')dQ_1(\theta')}\\
	&= \frac{\prod_{k=1}^tp_{\theta, k}(L_k)}{\int_{\theta' \in \Theta} \prod_{k=1}^t p_{\theta', k}(L_k)dQ_1(\theta')}
\end{align*}
where $q_{t+1}(\theta) = \frac{dQ_{t+1}}{dQ_1}(\theta)$ is the Radon-Nykodim derivative of the posterior measure with 
respect to the prior measure which is always guaranteed to exist (cf.~Theorem~1.31 of \citealp{schervish1996theory}).

As in the previous proof, we once again define $\overline{p}_t(L_t) = \int_{\theta\in\Theta}  
p_{\theta,t}(L_t)q_t(\theta)dQ_1(\theta)$ \\ and  $\overline{p}(L_{1:T}) = \int_{\theta\in\Theta}  \prod_{t=1}^T 
p_{\theta,t}(L_{t}) dQ_1(\theta)$, and compute the relation :  
\begin{align*}
	\overline{p}(L_{1:T}) &= \prod_{t=1}^T \frac{\overline{p}(L_{1:t})}{\overline{p}(L_{1:t-1})}
	= \prod_{t=1}^T \frac{\int_{\theta\in\Theta}  \prod_{k=1}^t p_{\theta,k}(L_k)dQ_1(\theta)}{\int_{\theta'\in \Theta} 
\prod_{k=1}^{t-1} p_{\theta',k}(L_k)dQ_{1}(\theta') 
		}
	\\
	&= \prod_{t=1}^T \int_{\theta\in\Theta} q_t(\theta) p_{\theta,t}(L_t)dQ_1(\theta)
	= \prod_{t=1}^T \overline{p}_t(L_t).
\end{align*}
Then, we have
\begin{align*}
	\log \frac{\prod_{t=1}^T p_{\theta^*,t}(L_{t})}{\prod_{t=1}^T \bp_t(L_t)}  &= \log \frac{\prod_{t=1}^T 
			p_{\theta^*,t}(L_{t})}{\overline{p}(L_{1:T})} 
	= \log \frac{\prod_{t=1}^T p_{\theta^*,t}(L_{t})}{\int_{\theta'\in\Theta}  \prod_{t=1}^T 
p_{\theta',t}(L_{t})dQ_1(\theta')}.
\end{align*}
Taking expectations and repeating the derivations of Equation~\eqref{eq:MI_vs_posterior_1} and 
Equation~\eqref{eq:MI_vs_posterior_2}, the proof is concluded.
\qed

\subsection{Proof of Theorem~\ref{th:main_continuous}}\label{app:th2}
The proof follows the same arguments as the proof of Theorem~\ref{th:main_discrete}, except we need a different  bound 
on the sum of information-gain terms that does not involve the entropy of $Q_1$. The bound is based on a covering 
argument that is provided by the following lemma, which, together with Equation~\eqref{eq:cs}, will 
directly imply the result claimed in the theorem.
\begin{lemma}\label{lem:cover}
Under the assumptions of Theorem~\ref{th:main_continuous},$
\EE{\sum_{t=1}^T I_t(\theta^\star;L_t)} \le \mathcal{N}_{\varepsilon}(\Theta_0,\varrho) + \varepsilon CT$.
\end{lemma}
\allowdisplaybreaks
\begin{proof}
	For the clarity of exposition, we still assume that the support $\Theta_1$ of $Q_1$ is a 
	countable subset of $\Theta$, even though our results can be extended to 
	general distributions.
	Fix an $\varepsilon>0$ and let $\wh{\Theta}_{\varepsilon}\subset \Theta$ be a minimal $\varepsilon$-cover of 
$\Theta_1$. Thus, for any $\theta \in 
	\Theta_1$, there exists $\wh{\theta} \in \wh{\Theta}_{\varepsilon}$ such that $\varrho(\theta,\htheta) \le 
\varepsilon$. 
	Let $\mathcal{G}_{\varepsilon}$ be the partition of $\Theta_1$ constructed by associating to each 
$\theta\in\Theta_1$ the closest (as measured by metric $\varrho$) element of the cover 
$\wh{\theta}\in\wh{\Theta}_{\varepsilon}$. For any $\theta\in\Theta_0$, we denote by 
$\mathcal{G}_{\varepsilon}(\theta)$ the unique set of the partition containing $\theta$. Note that, for any 
$\theta\in\Theta_1$ and $\ttheta\in\mathcal{G}_{\varepsilon}(\theta)$, $\rho(\theta,\ttheta)\le 2\varepsilon$ and so, 
by Assumption~\ref{ass:log_lipschitz}, $|\log p_{\theta,t}(L_t) - \log p_{\ttheta,t}(L_t)| \le  2C \varepsilon$.
	Hence, we have
	\begin{align*}
	-\log \pa{\sum_{\theta\in\Theta_1} q_1(\theta) \prod_{t=1}^T p_{\theta,t}(L_{t}) } &\le -\log 
	\pa{\sum_{\theta \in \mathcal{G}_\varepsilon(\theta^\star)} 
		q_1(\theta) \prod_{t=1}^T p_{\theta,t}(L_{t})}
	\\
	& = -\log \pa{\sum_{\theta \in \mathcal{G}_\varepsilon(\theta^\star)}  q_1(\theta) \prod_{t=1}^T 
		\pa{p_{\theta^\star,t}(L_{t})\cdot \frac{p_{\theta,t}(L_{t})}{p_{\theta^\star,t}(L_{t})}}}
	\\
	&\le -\log \pa{\sum_{\theta \in \mathcal{G}_\varepsilon(\theta^\star)}  q_1(\theta) \prod_{t=1}^T 
		\pa{p_{\theta^\star,t}(L_{t})\cdot e^{-2C\varepsilon}}}
	\\
	&\le -\log \pa{\sum_{\theta \in \mathcal{G}_\varepsilon(\theta^\star)}  q_1(\theta) \prod_{t=1}^T 
		p_{\theta^\star,t}(L_{t}) } + 2C\varepsilon T
	\\
	&= -\log \pa{\wh{Q}_1(\mathcal{G}_\varepsilon(\theta^\star)) \prod_{t=1}^T  p_{\theta^\star,t}(L_{t})} + 2C 
	\varepsilon T,
	\end{align*}
	where we denoted $\wh{Q}_1(\mathcal{G}_\varepsilon(\theta^\star)) = \sum_{\theta \in 
		\mathcal{G}_\varepsilon(\theta^\star)}  q_1(\theta) $. Using this result together with 
Lemma~\ref{lem:bayes_trick}, we 
	get
	\begin{align*}
	\EE{\sum_{t=1}^T g(\ell(\theta^\star, X_t, A_t) \| \bloss_t(X_t,A_t))}
	&=\EE{\log \frac{\prod_{t=1}^T p_{\theta^\star,t}(L_{t})}{\sum_{\theta\in\Theta_1} q_1(\theta) \prod_{t=1}^T 
			p_{\theta,t}(L_{t}) }}
	\\
	&\le \EE{ \log \frac{p_{\theta^\star} (L_{1:T})}{\wh{Q}_1(\mathcal{G}_\varepsilon(\theta^\star)) 
			p_{\theta^\star}(L_{1:T})}} + 2C \varepsilon T \\
	&= \EE{ - \log 
		\wh{Q}_1(\mathcal{G}_\varepsilon(\theta^\star))} + 2C \varepsilon T
	\\
	&=-\sum_{\theta\in\Theta_1}q_1(\theta) \log 
	\wh{Q}_1(\mathcal{G}_\varepsilon(\theta))  + 2C \varepsilon T\\
	&= - \sum_{G \in \mathcal{G}_{\varepsilon}} \wh{Q}_1(G)  \log 
	\wh{Q}_1(G)  + 2C\varepsilon T 
	\\
	&\le \log\pa{|\mathcal{G}_{\varepsilon}˘|} + 2C\varepsilon T.
	\end{align*}
	The proof of the lemma is concluded by noting that $|\mathcal{G}_\varepsilon|\le|\wh{\Theta}_{\epsilon}| \le 
\mathcal{N}_\varepsilon(\Theta_1,\varrho)$. 
\end{proof}

\subsection{The proofs for logistic bandits}\label{app:concrete}
\paragraph{Proof of Theorem~\ref{th:main_logistic}}\label{main_logistic_proof}
	Using a standard result on the covering number of the Euclidean ball, we have $\mathcal{N}_\varepsilon(\Theta_0,\norm{\cdot}) \le (\frac{2S}{\varepsilon}+1)^d$. 
	Regarding Assumption~\ref{ass:log_lipschitz}, notice that $\log(\sigma(f_{\theta}(x,a))) = 
	f_{\theta}(x,a) - \log(1+\exp(f_{\theta}(x,a)))$. Since $\log(\sigma(z))$ and $\log(1-\sigma(z))$ are $1$-Lipschitz and $f_{\theta}(x,a)$ is $C$-Lipschitz, 
	$\log P_{\theta,x,a}(1)=\log(\sigma(f_{\theta}(x,a)))$ and $\log P_{\theta,x,a}(0)=\log(1-\sigma(f_{\theta}(x,a)))$ are also $C$-Lipschitz, implying that 
	Assumption~\ref{ass:log_lipschitz} holds with the same constant $C$. 
	The claim then follows from Theorem~\ref{th:main_continuous} by taking $\varepsilon = 1/CT$.
\qed
\paragraph{Proof of Corollary~\ref{th:main_linear_logit}}\label{main_logistic_proof_col}
	The function $f_{\theta,x,a} : \theta \to \langle\varphi(x,a),\theta\rangle$ is linear in $\theta$ and since $\varphi(x,a)$ is bounded in norm by B, it is also $B$-Lipschitz.
	This proves that Assumption~\ref{ass:log_lipschitz} is satisfied with the constant $B$. We can then apply Theorem \ref{th:main_logistic} with $C=B$ and $\rho=2K$ (Lemma~\ref{lem:ir_finite}).
\qed

\subsection{The Legendre--Fenchel conjugate of the binary relative entropy}\label{app:aux}
\begin{proposition}\label{prop:dual_bound}
	For any $u\le 0$ and $q\in[0,1]$:
	\begin{equation*}
	g^*(u\|q) \le q\left(u + \frac{u^2}{2}\right).
	\end{equation*}
\end{proposition}
\begin{proof}
The claim follows from the following calculation:
	\begin{align}
	g^*(u\|q) = \log(1+q(e^u - 1)) \le q(e^u-1) \le q\left(u+\frac{u^2}{2}\right),
	\end{align}
	where the first inequality is from $\log(1+x)\le x$ for any $x>-1$, and the second inequality is from $e^x\le 
	1+x+\frac{x^2}{2}$ for any $x\le 0$.
\end{proof}

\section{Linear bandits}\label{app:linear}
By making an assumption that  loss function is a linear function of some feature map over states and actions, this linear dependence allows the learner to generalize the observation among the other actions. We provide two different results for the  linear bandits setting, the first one is Lemma~\ref{lem:ir_linear} that holds for general distributions with expected loss being a linear function taking values in $(0,1)$, and the other is Lemma~\ref{lem:ir_gaussian} that works for setting where the prior of $\theta^*$ is Gaussian distribution and the observations have a Gaussian noise. These two results are connected and share a similar proof techniques, thus we first state the following auxiliary lemma, that holds for general linear functions. We also need to define the following quantity:
\begin{equation}\label{eq:lir_linear}
	\delta_t = \frac{\Bpa{\EEt{\pa{ \bar{\ell}_t(X_t, A_t) - \ell(\theta_t, X_t, A_t)}}}^2}{\EEt{ \pa{\bar{\ell}_t(X_t, 
				A_t)  - \ell(\theta^*, X_t, A_t)  }^2}}.
\end{equation} 
We prove the following upper bound on this quantity:
\begin{lemma}\label{lm:lir_linear}
	Suppose the expected losses are linear functions of the form  $\ell(\theta, x, a) = \siprod{\theta}{\varphi(x,a)}$, for all $x\in \mathcal{X}, a\in \mathcal{A}, \theta \in \real^d$, where $\varphi:\mathcal{X}\times\mathcal{A}\ra \real^d$ is a feature 
	map. Then, $\delta_t \le d$.
\end{lemma}
\begin{proof}
	The proof follows from an adaptation of Proposition~5 of \citet{RvR16}. Their technique is based on constructing a 
	$K\times K$ matrix $M$ and showing that the information ratio can be upper bounded by the rank of this matrix.
	In particular, let us fix $t$ and define $\lambda_a = \pi_t(a|X_t)$ and the matrix $M$ whose elements are defined as
	\[
	M_{a,a'} = \sqrt{\lambda_a \lambda_{a'}}\pa{  \bar{\ell}_t(X_t, a)  - \EEcct{\ell(\theta^*, X_t, a) }{A_t^* = a'}}.
	\]
	First, we relate the denominator  of $\delta_t$ to this matrix as follows:
	\begin{align*}
		& \EEt{ \pa{\bar{\ell}_t(X_t, A_t)  - \ell(\theta^*, X_t, A_t)  }^2}= \EEt{\sum_a \pi_t(a|X_t) 
			\pa{\bar{\ell}_t(X_t, 
				a)  - \ell(\theta^*, X_t, a)  }^2}
		\\
		&\quad =   \sum_{a, a'} \pi_t(a|X_t) \pi_t(a'|X_t)    \EEcct{ \pa{\bar{\ell}_t(X_t, a)  - \ell(\theta^*, X_t, a)  
			}^2}{ A_t^* = a'} \nonumber
		\\
		&\quad \ge  \sum_{a, a'} \pi_t(a|X_t) \pi_t(a'|X_t) \Bpa{ \EEcct{\pa{\bar{\ell}_t(X_t, a)  - \ell(\theta^*, X_t, a)
			}}{A_t^* = a'} }^2\nonumber
		\\
		&\quad =   \sum_{a, a'} \pi_t(a|X_t) \pi_t(a'|X_t)  \Bpa{ \bar{\ell}_t(X_t, a)  - \EEcct{\ell(\theta^*, X_t, a)  
			}{A_t^* = a'} }^2 = \sum_{a,a'} M_{a,a'}^2 =  \norm{M}_F^2,\nonumber
	\end{align*}
	where we used Jensen's inequality. For the numerator, we get
	\begin{align*}
		&\EEt{ \bar{\ell}_t(X_t, A_t)  - \ell(\theta^*, X_t, A^*_t)  } =
		\EEt{ \bar{\ell}_t(X_t, A_t^*)  - \ell(\theta^*, X_t, A^*_t)  }
		\\
		&\qquad\qquad\qquad = \sum_a \pi_t(a|X_t) \pa{\bar{\ell}_t(X_t, a)  
			- \EEcct{\ell(\theta^*, X_t, a) }{A_t^* = a  } } = \trace{M}. 
	\end{align*}
	Thus, we have $\pa{\EEt{ \bar{\ell}_t(X_t, A_t)  - \ell(\theta^*, 
			X_t, A^*_t)  }}^2 = \bpa{\trace{M}}^2$. Putting this together with the previous bound, we conclude that $\rho_t \le 
	\pa{\trace{M}/\norm{M}_F}^2$, which can be further upper bounded by the rank of $M$. Now, one can follow 
	the steps of the proof of Proposition~5 of \citet{RvR16} to prove that the rank of $M$ is at most $d$.
\end{proof}
Equipped with result of Lemma~\ref{lm:lir_linear}, we can derive our analysis for the two described linear settings:
 
\paragraph{The proof of Lemma~\ref{lem:ir_linear}}\label{app:lir_linear}
We relate the denominator of $\delta_t$, defined in (\ref{eq:lir_linear}), to the denominator of $\rho_t$, by using Pinsker's inequality as follows:
\begin{align*}\label{lower_boundIG}
	&I_t(\theta^*;L_t)\ge 2 \EEt{ \pa{\bar{\ell}_t(X_t, A_t)  - \ell(\theta^*, X_t, A_t)  }^2}
\end{align*}
 Using the definition of $\rho_t$ and Lemma~\ref{lm:lir_linear}, we get that $\rho_t \le \frac{d}{2}$.
\qed

\paragraph{The proof of Lemma~\ref{lem:ir_gaussian}}\label{app:gaussian}
Note that for the Gaussian likelihood, $I_t(\theta^*;L_t) $ $= \EEt{ 
	\pa{\bar{\ell}_t(X_t, A_t)  - \ell(\theta^*, X_t, A_t)  }^2}$,
so  we have
\begin{align*}
	\rho_t = \frac{\Bpa{\EEt{\pa{ \bar{\ell}_t(X_t, A_t) - \ell(\theta_t, X_t, A_t)}}}^2}{\EEt{ \pa{\bar{\ell}_t(X_t, 
				A_t)  - \ell(\theta^*, X_t, A_t)  }^2}}.
\end{align*} 
Lemma~\ref{lm:lir_linear} provides $\rho_t  = \delta_t \le d$ bound. We still can tighten this results for the cases when $\log K \ll d$ by using the linearity and the closed form expression for $I_t(\theta^*;L_t)$. 
Under the stated assumptions, the 
posterior $Q_{t}$ is $\mathcal{N}( \overline{\theta},   \overline{\Sigma}_{t})$ with $\overline{\theta}_{t} = 
\overline{\Sigma}_{t}^{-1}\sum_{s<t} \varphi(X_s, A_s)\varphi(X_s, A_s)\transpose L_s $.  and $ \overline{\Sigma}_{t} = 
\sum_{s<t} \varphi(X_s, A_s)\varphi(X_s, A_s)\transpose + \frac{\sigma^2}{\lambda} I$.  
Note that $\bar \ell _t(x, a) =   \EEt{\iprod{\theta}{\varphi(x, a)}} =\iprod{\overline{\theta}_t}{\varphi(x, a)}$.
We first consider the numerator of $\rho_t$. Let us define $\sigma_{t,a}^2 = \varphi(X_t, a)\transpose \Sigma_{t}^{-1} 
\varphi(X_t, a) $,
and notice that conditional on $\mathcal{F}_t, X_t$, we have $\siprod{\theta_{t} - \hat{\theta}_{t}}{\varphi(X_t, 
a)}  
\sigma_{t,a}^{-1} \sim \mathcal{N}(0,1)$. Let $Z$ denote a standard normal, then for $c>0$, $\PP{Z > c}  \le 0.5 e^{-c^2/2}. $ Using $c_{\delta} = 
\sqrt{2 \log(K/\delta)}$, we can show
\[ \PPcct{ |  \varphi(X_t, a)\transpose (\theta_{t} - \hat{\theta}_{t})  \sigma_{t,a}^{-1} | > c_{\delta} \text{ for 
all }a }{ X_t, A_t } \le \delta. \] 
Then, with probability at least $1-\delta$, 
\begin{align*}
	&\pa{\ell _t(X_t, A_t) - \ell(\theta_t, X_t, A_t)}^2\sigma^{-2}_{t,A_t} \le c_{\delta}. 
\end{align*}
Note that:  \[\EEcct{ \pa{\ell _t(X_t, A_t) - \ell(\theta_t, X_t, A_t)}^2 \sigma^{-2}_{t,A_t} }{X_t, A_t} =   \frac{\EEcct{ \pa{\ell _t(X_t, A_t) - \ell(\theta_t, X_t, A_t)}^2 }{X_t, A_t}}{\sigma^{2}_{t,A_t}}.\] In order to get the bound on this, we will use the idea of proof of Theorem 3.3 from \cite{BCB12}: taking $W_t = \frac{ \pa{\ell _t(X_t, A_t) - \ell(\theta_t, X_t, A_t)}^2}{2\sigma^2_{t,A_t} } - \ln K$ and using the formula \[\EEcct{W}{X_t, A_t} \le \int_0^1 \frac{1}{\delta} \PPcct{W > \ln \frac{1}{\delta}}{X_t, A_t}d\delta,\]
we get $\EEt{W} \le 1$, which gives
\[\EEcct{\pa{\ell _t(X_t, A_t) - \ell(\theta_t, X_t, A_t)}^2\sigma^{-2}_{t,A_t} }{X_t, A_t} \le  2(1+\ln K)\sigma^2_{t,A_t}  . \]
Now, we notice that the denominator can be written as 
\begin{equation}\label{eq:gauss_variance}
	\EEt{ \pa{\bar{\ell}_t(X_t, A_t)  - \ell(\theta^*, X_t, A_t)  
		}^2} = \EEt{  \varphi(X_t, A_t)\transpose \Sigma_{t}^{-1} \varphi(X_t, A_t)   } = \EEt{\sigma^2_{t,A_t} }.
\end{equation}
Therefore, we can bound $\rho_t$ for all $t$ as follows:
\begin{align*}
		\rho_t \le \frac{\EEt{\pa{ \bar{\ell}_t(X_t, A_t) - \ell(\theta_t, X_t, A_t)}^2}}{\EEt{ \pa{\bar{\ell}_t(X_t, A_t)  - \ell(\theta^*, X_t, A_t)  }^2}} \le \frac{2(1+\ln K)  \EEt{  \sigma^2_{t,A_t}}}{\EEt{   \sigma^2_{t,A_t}  }} = 2(1+\ln K).
\end{align*}

\qed

For the complete of our analysis, we also prove a regret bound for the Gaussian setting:
\begin{cor}\label{gp_prior}
	Consider the classic setting of linear bandits, such that the loss of playing action $a$ at time $t$ satisfies 
	$\ell(\theta^*, x, a) \sim \mathcal{N}(\langle 
	\theta^{\star},\varphi(x,a)\rangle, \sigma^2)$, where $\theta^*$ is drawn from a Gaussian prior
	$Q_1 =  \mathcal{N}(0, \lambda I )$ with some $\lambda >0$, $\sigma >0$  and $\varphi:\X\times[K]\to\Reals^d$ is a feature map 	such that $ 
	\norm{\varphi(x,a)} \le B$ for all $x\in \X, a \in [K]$.  Then, the Bayesian regret of TS satisfies
	\[ \EE{R_T} \le \sqrt{ 2d T  \min\{2(1+\log K),d\}  \log\pa{1+ \frac{ T\lambda B^2 }{d \sigma^2}}}   .\]
\end{cor}
\begin{proof}
By using the form of $I_t(\theta^*;L_t) $ for the Gaussian likelihood and applying the regret decomposition of Equation (\ref{eq:cs}), we get 
	\begin{align}\label{decomp_square}
		&\EE{R_T} \le \sqrt{ \sum_{t=1}^T \EEt{\rho_t} \EE{   \sum_{t=1}^T \pa{\bar{\ell}_t(X_t, A_t)  - \ell(\theta^*, X_t, A_t)  }^2} }.
	\end{align}

Then, by Lemma~\ref{lem:ir_gaussian} and (\ref{eq:gauss_variance}),
\[ \EE{R_T} \le \sqrt{ \min\{2   (1+\ln K) ,d\}T\sum_{t=1}^{T} \EEt{  \varphi(X_t, A_t)\transpose \Sigma_{t}^{-1} \varphi(X_t, A_t)   }}   \]

The term $ \sum_{t=1}^{T} \EEt{  \varphi(X_t, A_t)\transpose \Sigma_{t}^{-1} \varphi(X_t, A_t)   }$ can be addressed by using the elliptical potential lemma (e.g. Lemma 19.4. in \cite{LSz19book}), which is widely used in the analysis of linear bandits. Applying this lemma, we get 

\begin{align*}
	\sum_{t=1}^{T}    \varphi(X_t, A_t)\transpose \Sigma_{t}^{-1} \varphi(X_t, A_t) \le 2d  \log\pa{1+ \frac{ T\lambda B^2 }{d \sigma^2}},
\end{align*}
which completes the proof. 
\end{proof}

\end{document}